\RequirePackage{fix-cm}
\documentclass[smallcondensed]{svjour3}     % onecolumn (ditto)
\smartqed  % flush right qed marks, e.g. at end of proof
\usepackage{graphicx}
\usepackage{natbib}
\setcitestyle{aysep={}, round, semicolon, yysep={,}}
\setcitestyle{notesep={: }}
\usepackage{amsmath}
\usepackage{url}
\usepackage{amsfonts}
\usepackage{bm}
\usepackage{geometry}
\usepackage{setspace}
\usepackage{tabu}
\usepackage{multicol}
\usepackage{multirow}
\usepackage{booktabs}
\usepackage{cite}
\usepackage{booktabs}
\usepackage{threeparttable}
\usepackage{xcolor}
\usepackage{appendix}
\usepackage{pgf}
\usepackage{tikz}
\usepackage[utf8]{inputenc}
\usetikzlibrary{arrows,automata}
\usetikzlibrary{positioning}
\tikzset{
    state/.style={
           rectangle,
           rounded corners,
           draw=black, very thick,
           minimum height=2em,
           inner sep=2pt,
           text centered,
           },
}

\makeatletter
\newif\if@restonecol
\makeatother

\usepackage[linesnumbered,ruled,vlined]{algorithm2e}%[ruled,vlined]{
\usepackage{algpseudocode}
\newcommand{\bit}{\begin{itemize}}
\newcommand{\eit}{\end{itemize}}
\newcommand{\be}{\begin{equation}}
\newcommand{\ee}{\end{equation}}
\newcommand{\RNum}[1]{\uppercase\expandafter{\romannumeral #1\relax}}
\spnewtheorem{assumption}{Assumption}{\bf}{\it}
\journalname{Machine Learning}
\begin{document}

\title{A Semismooth-Newton's-Method-Based Linearization and Approximation Approach for Kernel Support Vector Machines%\thanks{Grants or other notes
}


\author{Chen Jiang \and Qingna Li}


\institute{C. Jiang \at
              School of Mathematics and Statistics, Beijing Institute of Technology. Beijing, 100081, P. R. China. \\
              \email{jiangchenwo@gmail.com}           %  \\
           \and
           Q. Li \at
              Corresponding author. School of Mathematics and Statistics/Beijing Key Laboratory on MCAACI, Beijing Institute of Technology. Beijing, 100081, P. R. China. \\
              \email{qnl@bit.edu.cn}. \\
              This author's research was supported by the National Natural Science Foundation of China (No.11671036).
}

\date{Received: date / Accepted: date}

\maketitle

\begin{abstract}
Support Vector Machines (SVMs) are among the most popular and the best performing classification algorithms. Various approaches have been proposed to reduce the high computation and memory cost when training and predicting based on large-scale datasets with kernel SVMs. A popular one is the linearization framework, which successfully builds a bridge between the $L_1$-loss kernel SVM and the $L_1$-loss linear SVM. For linear SVMs, very recently, a semismooth Newton's method is proposed. It is shown to be very competitive and have low computational cost. Consequently, a natural question is whether it is possible to develop a fast semismooth Newton's algorithm for kernel SVMs. Motivated by this question and the idea in linearization framework, in this paper, we focus on the $L_2$-loss kernel SVM and propose a semismooth Newton's method based linearization and approximation approach for it. The main idea of this approach is to first set up an equivalent linear SVM, then apply the Nystr\"om method to approximate the kernel matrix, based on which a reduced linear SVM is obtained. Finally, the fast semismooth Newton's method is employed to solve the reduced linear SVM. We also provide some theoretical analyses on the approximation of the kernel matrix. The advantage of the proposed approach is that it maintains low computational cost and keeps a fast convergence rate. Results of extensive numerical experiments verify the efficiency of the proposed approach in terms of both predicting accuracy and speed.
\keywords{Support vector machines \and Kernel methods \and Semismooth Newton's methods \and Nystr\"om methods}
\end{abstract}

\section{Introduction}
\label{intro}
Support Vector Machines (SVMs) \citep{cortes1995support,xie2019efficient} are among the most popular and the best performing classification algorithms. SVMs have been successfully used in various applications such as test classification \citep{tong2001support,zhang2008text}, computational biology \citep{scholkopf2004support,huang2018applications} and finance \citep{chen2017feature}. For data with linear boundaries, the linear SVMs aim to generate an optimal separating hyperplane between the two classes. Kernel methods \citep{scholkopf2002learning} map the input data into the reproducing kernel Hilbert space (RKHS), which allow Kernel SVMs to abstract the nonlinear relations in the input data. While kernel SVMs provide powerful tools to solve classification problems with various input data, there are also challenges in designing algorithms for kernel SVMs. The first challenge is how to compute and save the kernel matrix, which is usually dense \citep{shin2005invariance,feng2017scalable}. In addition to that, in kernel SVMs, the number of support vectors that have to be explicitly maintained grows linearly with the sample size on noisy data, which is referred to as the curse of kernelization \citep{wang2010multi}. 

Various approaches have been proposed to reduce the high computation and memory cost when training and predicting based on large-scale datasets with kernel SVMs such as SVMperf \citep{joachims2009sparse}, Pegasos \citep{shalev2011pegasos}, budgeted stochastic gradient descent (BSGD) \citep{wang2012breaking,djuric2013budgetedsvm} and the widely used LIBSVM \citep{CC01a}. However, due to the data explosion in the past few years, efficient algorithms for large-scale kernel SVMs are still highly in need.

{\bf Related Works.} One popular way to deal with large-scale kernel SVMs is the linearization framework \citep{zhang2012scaling}, which successfully builds a bridge between the $L_1$-loss kernel SVM and the $L_1$-loss linear SVM. The linearization framework enables us to linearize the kernel SVM through approximation and decomposition of the kernel matrix and solving it with linear solvers, so that solving large-scale kernel SVMs can also benefit from the advanced and extremely efficient linear SVMs' solvers. Inspired by the idea of \citet{zhang2012scaling}, efforts have been made to improve the approximation of the kernel matrix under the linearization framework, such as memory efficient kernel approximation \citep{si2017memory} and Hash-SVM \citep{mu2014hash}. 

One of the attractive properties of the linearization framework is that it provides us a way to solve kernel SVMs by various linear SVMs' solvers. A variety of methods have been proposed, including the popular trust region Newton method (TRON) \citep{lin2008trust} and the dual coordinate descent method (DCD) \citep{hsieh2008dual}. Recently, there has been great progress on algorithms for linear SVMs with large-scale datasets \citep{yuan2012recent}, for instance the stochastic gradient descent method \citep{zhang2004solving}, the cutting plane method \citep{joachims2006training}, SVM-ALM algorithm \citep{nie2014new}, the fast APG (FAPG) method \citep{ito2017unified} as well as the recent AL-SNCG method \citep{yan2019alm}. We refer to \citet{chauhan2019problem} for monographs and reviews on linear SVMs. In particular, a semismooth Newton's method\citep{yin2019semismooth} is proposed very recently, which is shown to be very competitive and have low computational cost. Consequently, a natural question is whether it is possible to develop a fast semismooth Newton's algorithm for kernel SVMs.

{\bf Our Contributions.} Motivated by this question and the idea in linearization framework, in this paper, we focus on the $L_2$-loss kernel SVM and propose a semismooth Newton's method based linearization and approximation approach for it. The main idea of this approach is to first set up an equivalent linear SVM, then apply the Nystr\"om method to approximate the kernel matrix, based on which a reduced linear SVM is obtained. We also provide some theoretical analyses on the approximation of the kernel matrix. Finally, the fast semismooth Newton's method is employed to solve the reduced linear SVM. The advantage of the proposed approach is that it maintains low computational cost and keeps a fast convergence rate. Results of extensive numerical experiments verify the efficiency of the proposed approach in terms of both predicting accuracy and speed.

The rest of the paper is organized as follows. In Section \ref{sec-models}, we present the kernel SVMs and a brief review about linearized kernel SVM proposed by \citet{zhang2012scaling}. In Section \ref{sec-method}, we introduce the linearization and approximation approach. In Section \ref{sec-error}, we analyze the theoretical error of approximation for the kernel matrix. In Section \ref{sec-subproblem}, we apply the fast semismooth Newton's method for the reduced linear SVMs. In Section \ref{sec-numerical}, we conduct numerical tests to verify the efficiency of our approach. Final conclusions are given in Section \ref{sec-conclusions}.

\paragraph{Notations} We use bold letters to indicate vectors and matrices, and $\|\cdot\|$ to denote the $l_2$ norm for vectors and Frobenius norm for matrices. Let $\mathbb S^n$ denote the space of $n\times n$ symmetric matrices.

\section{Preliminaries}\label{sec-models}

\subsection{Kernel SVMs}
The SVMs can be divided into support vector machine classifiers and support vector regression (SVR) models due to different purposes. In our paper, we focus on SVM classifiers and our method is also applicable to SVR models.

Given training data consists of $n$ pairs $(\mathbf{x}_1, y_1) ,(\mathbf{x}_2, y_2) ,\cdots,(\mathbf{x}_n, y_n) $, with $\mathbf{x}_i \in \mathbb{R}^p$ and $y_i \in \{-1,1\}$, the idea of kernel SVMs is to map the training data from the input space $\mathbb{R}^p$ to a Hilbert space $\mathbb{Y}$ by a feature mapping function $\bm \psi: \mathbb R^p\rightarrow \mathbb Y$, where $\mathbb Y$ is the feature space. The kernel SVMs are to train the following model
\begin{equation}%\label{SVM-general}
\min_{\mathbf{w}  , b  } \quad \frac{1}{2}\|\mathbf{w} \|^2+ C_0\widehat{\mathcal R}(\mathbf w,b),
\label{eq:svc0-kernel}
\end{equation}
where $C_0$ is some positive constant, $\widehat{\mathcal R}(\mathbf w,b) = \frac{1}{n}\sum_{i=1}^{n}L(\mathbf w,b;\bm \psi(\mathbf x_i),y_i)$ is the empirical error, with $L(\cdot)$ being the loss function. Denote that $C = C_0/n$ is the cost parameter. A special case is that when $\bm \psi$ is an identity  mapping, i.e., $\bm \psi(x)=x$, then kernel SVMs (\ref{eq:svc0-kernel}) reduce to the linear SVMs.

Denote
\be\label{eq:data}
\mathbf X_r =\left[
\begin{array}{c}
\mathbf x_1^\top \\
\vdots\\
\mathbf x_n^\top
\end{array}
\right]\in\mathbb R^{n\times p}, \ \
\mathbf y_r =\left[
\begin{array}{c}
 y_1\\
\vdots\\
 y_n
\end{array}
\right]\in\mathbb R^{n}, \
\mathbf X_e =\left[
\begin{array}{c}
\mathbf x_{n+1}^\top \\
\vdots\\
\mathbf x_{n+m}^\top
\end{array}
\right]\in\mathbb R^{m\times p},
\ee
where $\mathbf X_e$ is the test dataset.
Three popular choices for $L(\cdot)$ are the $L_1$-loss function, $L_2$-loss function and logistic function. In our paper, we focus on the $L_2$-loss kernel SVM, i.e.,
 \be\label{eq:svc-l2-kernel}
\min_{\mathbf w, b}\ \frac12\|\mathbf w\|^2+C\sum_{i = 1}^n \max(0, y_i(\bm \psi(\mathbf x_i)^\top \mathbf w + b))^2.
 \ee
 It can be equivalently written as
 \be\label{eq:l2-svc-kernel}
\renewcommand\arraystretch{1.2}
\begin{array}{ll}
\min\limits_{\mathbf{w}  , b } & \frac{1}{2}\|\mathbf{w} \|^2+ C\sum_{i=1}^{n}\xi_i^2\\
\hbox{s.t.} & y_i (\mathbf{w}^\top \bm \psi(\mathbf{x}_i )+b ) \geq 1-\xi_i ,\  \xi_i \geq 0, \ i = 1, \dots, n.
\end{array}
%$$y_i (\mathbf{w}^\top \mathbf{x}_i +b ) \geq 1-\xi_i ,\  \xi_i \geq 0$$
\ee
 with the dual problem
 \be\label{eq:dual-l2-kernel}
\renewcommand\arraystretch{1.2}
\begin{array}{ll}
\max\limits_{\bm \lambda\in\mathbb R^n}& \bm\lambda^\top e -\frac12\bm \lambda^\top \mathbf Q\bm\lambda-\frac1{4C}\|{\bm\lambda}\|^2\\
\hbox{s.t.}& \bm \lambda^\top \mathbf y = 0,\ \bm\lambda\ge0,\\
%\hbox{ }& 0 \mathbf\alpha.
\end{array}
\ee
where $\mathbf{Q}\in\mathbb S^{n}$ is defined by $\mathbf \overline Q_{ij}=y_iy_j\langle\bm \psi(\mathbf x_i),\bm \psi(\mathbf x_j)\rangle$, $i , j = 1, \dots,n$.

Let $(\mathbf w^*,b^*)$ be the optimal solution of (\ref{eq:svc0-kernel}). The predicting label for testing data $\mathbf x\in\mathbb R^p$ is
\be\label{pre-kernel}
\hbox{sign}((\mathbf w^*)^\top \bm \psi(\mathbf x) + b^*).
\ee
%To summarize, we use $(\mathbf w^\kappa, b^\kappa, )$
Let $\bm \lambda^*$ be the optimal solution of (\ref{eq:dual-l2-kernel}), there is $\mathbf w^*=\sum_{i}\lambda_i y_i\bm \psi(\mathbf x_i)$, and the predicting label becomes
\be\label{eq:dual-l2-kernel-label}
\hbox{sign}(\sum_{i}y_i\lambda_i\langle\bm \psi(\mathbf x_i),\bm \psi(\mathbf x )\rangle+ b^*).
\ee
Given the fact that $\bm \psi$ may be an infinite mapping, it may not be easy to give $\bm\psi$ explicitly. Since (\ref{eq:dual-l2-kernel}) and (\ref{eq:dual-l2-kernel-label}) involve $\bm \psi(\mathbf x)$ only through the inner product, one can define the kernel function $\kappa: \mathbb R^p\times \mathbb R^p\to \mathbb R$ instead, by
$\kappa(\mathbf x, \mathbf x') = \langle \bm \psi(\mathbf x), \bm \psi(\mathbf x')\rangle$. Popular kernel functions\citep{hastie2005elements} include
\bit
\item
$d${th-degree polynomial:}  $ \kappa(\mathbf x, \mathbf x') = (1 + \langle \mathbf x, \mathbf x'\rangle)^d$,
\item
{radial basis:} $\kappa(\mathbf x, \mathbf x') = \exp(-\gamma\|\mathbf x-\mathbf x'\|^2)$,
\item
{neural network:} $
\kappa(\mathbf x, \mathbf x') = \tanh(\alpha \langle \mathbf x, \mathbf x'\rangle+ \beta)$.
\eit
Once the kernel function $\kappa$ is given, methods designed for solving dual problems of linear SVMs can be easily adapted to solve corresponding dual problems of kernel SVMs. But such extensions for methods that is designed to solve primal forms of linear SVMs are not trivial.

\citet{zhang2012scaling} have proposed a linearization approach for the $L_1$-loss kernel SVM (\ref{eq:l1-svc}), which is briefly reviewed below.

\subsection{Linearized $L_1$-loss kernel SVM}

Define the positive semidefinite kernel matrix $\mathbf K\in\mathbb R^{(n+m)\times (n+m)}$ as \[
K_{ij}=\langle \bm\psi(\mathbf x_i), \bm\psi(\mathbf x_j)\rangle, \ i,j = 1,\dots, m+n.
\]
Rewrite $\mathbf K$ in the following partition
\begin{equation}\label{eq:K}
\mathbf{K} =
 \left[
\begin{array}{cc}
\mathbf{K}_{rr} & \mathbf{K}_{er}\\
\mathbf{K}_{re} & \mathbf{K}_{ee}\\
\end{array}
\right] \ \hbox{with} \ \mathbf K_{rr}\in \mathbb S^n, \ \mathbf K_{ee}\in \mathbb S^m, \mathbf K_{re} \in\mathbb R^{n\times m}.
\end{equation}
The following result comes form Proposition 1 in \citep{zhang2012scaling}, which addresses the method of transforming the $L_1$-loss kernel SVM 
\be\label{eq:l1-svc-kernel}
\renewcommand\arraystretch{1.2}
\begin{array}{ll}
\min\limits_{\mathbf{w}  , b  } & \frac{1}{2}\|\mathbf{w} \|^2+ C\sum_{i=1}^{n}\xi_i\\
\hbox{s.t.} & y_i (\mathbf{w}^\top \bm \psi(\mathbf{x}_i )+b ) \geq 1-\xi_i ,\  \xi_i \geq 0, \ i = 1, \dots, n.
\end{array}
%$$y_i (\mathbf{w}^\top \mathbf{x}_i +b ) \geq 1-\xi_i ,\  \xi_i \geq 0$$
\ee
into the $L_1$-loss linear SVM 
\be\label{eq:l1-svc}
\begin{array}{ll}
\min\limits_{\mathbf{w} , b } & \frac{1}{2}\|\mathbf{w} \|^2+ C\sum_{i=1}^{n}\xi_i
\\
\hbox{s.t.} & y_i (\mathbf{w}^\top \hat{\mathbf{x}}_i +b ) \geq 1-\xi_i ,\  \xi_i \geq 0, \ i = 1, \dots, n.
\end{array}
%$$y_i (\mathbf{w}^\top \mathbf{x}_i +b ) \geq 1-\xi_i ,\  \xi_i \geq 0$$
\ee

\begin{proposition}\label{prop-1}
Given the training data $\mathbf{X}_r$, label $\mathbf{y}_r$ and test data $\mathbf{X}_e$ as defined in (\ref{eq:data}). An $L_1$-loss kernel SVM model (\ref{eq:l1-svc-kernel}) trained on $\mathbf{X}_r$, $\mathbf{y}_r$ and tested on $\mathbf{X}_e$ is equivalent to a linear SVM (\ref{eq:l1-svc}) trained on $\mathbf{F}_r$, $\mathbf{y}_r$ and tested on $\mathbf{F}_e$, where
\begin{equation}\label{eq:K-FrFe}
\mathbf{K} =
\left[
\begin{array}{c}
\mathbf{F}_r\\
\mathbf{F}_e\\
\end{array}
\right]
\left[
\begin{array}{cc}
\mathbf{F}_r^\top &\mathbf{F}_e^\top\\
\end{array}
\right], \ F_r = \left[
\begin{array}{c}
\hat {\mathbf x}_1^T\\
\vdots\\
\hat{\mathbf x}_n^T
\end{array}
\right], \
 F_e = \left[
\begin{array}{c}
\hat {\mathbf x}_{n+1}^T\\
\vdots\\
\hat{\mathbf x}_{n+m}^T
\end{array}
\right]
\end{equation}
is any decomposition of the positive semidefinite kernel matrix $\mathbf{K}$ evaluated on $(\mathbf{X}_r,\mathbf{X}_e)$, and the factor $\mathbf{F}_r \in\mathbb{R}^{n\times q}$ and $\mathbf{F}_e \in \mathbb{R}^{m\times q}$ can be deemed as "virtual samples" whose dimensionality $q$ is the rank of $\mathbf{K}$.
\end{proposition}

By Proposition \ref{prop-1}, Zhang et. al. proposed a framework to solve the $L_1$-loss kernel SVM, by solving the $L_1$-loss linear SVM. Note that the linearization process may not be easy, for instance the exact spectral decomposition of the kernel matrix $\mathbf K$ takes $O(n^3)$ operations. Consequently, an approximation is further conducted by using the Nystr{\"o}m methods \citep{williams2001using}, which approximates $\mathbf F_r\in\mathbb R^{n\times q}$  by $\widetilde {\mathbf F}_r\in\mathbb R^{n\times k}$. Finally, an $L_1$-loss linear SVM model is trained on $\widetilde{\mathbf F}_r$.

 The idea in Proposition \ref{prop-1} provides us a way to make use of fast solvers in linear SVM. As we mentioned in Introduction, one of the latest fast solvers is a semismooth Newton's method \citep{yin2019semismooth} for $L_2$-loss linear SVM
 \be\label{eq:svc-l2-linear}
 \min_{\mathbf w, \ b}\ \frac12\|\mathbf w\|^2+C\sum_{i = 1}^n \max(0, 1-y_i(\mathbf x_i^\top \mathbf w + b))^2.
 \ee
It is demonstrated by \citet{yin2019semismooth} that semismooth Newton's method is competitive with DCD and TRON in LIBLINEAR. Inspired by \citet{zhang2012scaling}, we can explore the technique in Proposition \ref{prop-1}, and extend semismooth Newton's method \citep{yin2019semismooth} to solve the $L_2$-loss SVM. We state our approach in the following section.

\section{A linearization and approximation approach}\label{sec-method}
In this section, we first get the equivalence of the $L_2$-loss kernel SVM and the $L_2$-loss linear SVM with some relationship between their data, similar to the way in Proposition \ref{prop-1}. Then we apply the Nystr{\"o}m method \citep{williams2001using} to get an approximation of the kernel matrix, based on which a reduced $L_2$-loss linear SVM is obtained.

\subsection{Equivalent linear SVM}
 Similar to Proposition \ref{prop-1}, we have following result, whose proof is similar to that of Proposition \ref{prop-1}. For consideration of completion, we include our proof in Appendix A.

\begin{theorem}\label{thm-2}
An $L_2$-loss SVM (\ref{eq:svc-l2-kernel}) trained on $\mathbf{X}_r$, $\mathbf{y}_r$ and tested on $\mathbf{X}_e$ is equivalent to a linear $L_2$-loss SVM (\ref{eq:svc-l2-linear}) trained on $\mathbf{F}_r$, $\mathbf{y}_r$ and tested on $\mathbf{F}_e$, where $\mathbf K$ is defined as in (\ref{eq:K}) and $\mathbf {F}_r, \mathbf{F}_e$ are defined as in (\ref{eq:K-FrFe}).
\end{theorem}

\subsection{Low-rank approximation of the kernel matrix}
By Theorem \ref{thm-2}, solving $L_2$-loss kernel SVM (\ref{eq:svc-l2-kernel}) is equivalent to solving (\ref{eq:svc-l2-linear}), and the predicting label for $\hat{\mathbf x}_i$ is given by
\be
\hbox{sign}(\hat{\mathbf x}_i^\top \mathbf w^* + b^*),\  i =n+ 1, \dots,n+m.
\ee
Now the key is to find $\mathbf{F}_r, \mathbf{F}_e$ such that (\ref{eq:K-FrFe}) holds. Since kernel matrix $\mathbf K_{rr}$ is semidefinite, an obvious way is just to use the spectral decomposition of $\mathbf K_{rr}$, however the computation cost is high as we mentioned before. Consequently, it is a good choice to approximate the kernel matrix, such as using low-rank approximation. Consider solving the following optimization problem
\be
\min_{\mathbf M\in\mathbb S^{n}}\ \frac12 \|\mathbf{K}_{rr}-\mathbf M\|, \hbox{s.t.}\ \hbox{rank}(\mathbf M)\le k, \ \mathbf M \hbox{ is positive semidefinite,}
\ee
the solution is denoted as $\mathbf {K}_{rr}^{(k)}$. Then 
$$\mathbf{K}^{(k)}_{rr}= \mathbf{F}^{(k)}_{r}(\mathbf{F}^{(k)}_{r})^\top,$$
where
\be
\label{eq:fk}
\mathbf{F}^{(k)}_{r}= \mathbf U^{(k)}_{r}(\bm{\Lambda}^{(k)}_{r})^{1/2}
\ee
and $\bm{\Lambda}^{(k)}_{r}$ is a diagonal matrix with diagonal entries being top k eigenvalues of $\mathbf{K}_{rr}$ and $\mathbf U^{(k)}_{r}$ stands for corresponding eigenvectors. 

However, it is not applicable when the kernel SVM is trained on a dataset with thousands of data points. In fact, exact computation of the top $k$ eigenvectors requires $O(n^2k)$ time and $O(n^2)$ space, which could be extremely time consuming. %Therefore we are seeking for other kinds of the low rank approximation of kernel matrix, which has been well studied in recent years.
Another popular approximation method, which has been well studied recently, is Nytstr{\"o}m method \citep{williams2001using,kumar2009sampling}. Given a set of training samples $\mathbf{X}_r$, a set of testing samples $\mathbf{X}_e$ and the kernel matrix $\mathbf{K}$ that is defined as in (\ref{eq:K}), the Nystr{\"o}m method chooses a subset of $k$ samples $\mathbf{L}\in\mathbb R^{k\times q}$, named landmark points set, from training samples $\mathbf{X}_r$ and provides a rank-$k$ approximation of the kernel matrix as
\begin{equation*}
\begin{array}{ll}
\mathbf{\widetilde K}_{rr} = \mathbf{K}_{rl} \mathbf{K}_{ll}^{-1}\mathbf{K}_{rl}^\top ,\
\mathbf{\widetilde K}_{ee} = \mathbf{K}_{el} \mathbf{K}_{ll}^{-1}\mathbf{K}_{el}^\top ,
\end{array}
\end{equation*}
where $\mathbf{K}_{rl}$ is the kernel matrix on $\mathbf{X}_r$ and $ \mathbf{L}$, $\mathbf{K}_{el}$ is the kernel matrix on $\mathbf{X}_e$ and $ \mathbf{L}$, $\mathbf{K}_{ll}$ is the kernel matrix on $ \mathbf{L}$ \citep[Eq.(10)]{williams2001using}. We refer to \citet{williams2001using} for more details of Nystr{\"o}m method.

Consequently, assume that the spectral decomposition of $\mathbf{K}_{ll}$ is
\[\mathbf{K}_{ll}= \mathbf{V}\bm{\Lambda}\mathbf{V}^{-1},\]
 where $\mathbf K_{ll}$ is positive definite. Let $\mathbf{M}= \mathbf{V}\bm{\Lambda}^{-\frac{1}{2}}$, referred to as the mapping matrix, then
 %{\color{red}(Here and in the following, all the $* $replaced by $\tilde{\cdot}$. For example, $K^*_{rr}$ replaced by $\widetilde K_{rr}$)}
\[
\mathbf{\widetilde K}_{rr} = \mathbf{K}_{rl}\mathbf{V}\bm{\Lambda}^{-\frac{1}{2}}(\mathbf{K}_{rl}\mathbf{V}\bm{\Lambda}^{-\frac{1}{2}})^\top =\mathbf{K}_{rl}\mathbf{M}(\mathbf{K}_{rl}\mathbf{M})^\top .\]
Similarly we have
\[ \mathbf{\widetilde K}_{ee} = \mathbf{K}_{el}\mathbf{M}(\mathbf{K}_{el}\mathbf{M})^\top .\]
Denote
\be
\label{eq:fkn}
\mathbf{\widetilde F}_r = \mathbf{K}_{rl}\mathbf{M},\ \
\mathbf{\widetilde F}_e = \mathbf{K}_{el}\mathbf{M}.
\ee
Instead of solving $L_2$-loss kernel SVM model (\ref{eq:svc-l2-kernel}), we can train the following linear SVM on $ \widetilde {\mathbf F}_r, \ \mathbf y_r$ and test on $\widetilde {\mathbf F}_e$
\begin{equation}\label{eq:svc-l2-linear-reduce}
\begin{array}{ll}
\min\limits_{\mathbf{w}  , b  } & \frac{1}{2}\|\mathbf{w} \|^2+ C\sum_{i=1}^{l}\xi_i^2\\
 \hbox{s.t.} &y_i (\mathbf{w}^\top \tilde{\mathbf{x}}_{i}+b ) \geq 1-\xi_i ,\  \xi_i \geq 0
\end{array}
\end{equation}
where \[
\widetilde{\mathbf F}_r:=\left[
\begin{array}{l}
\tilde{\mathbf x}_{1}\\
\vdots\\
\tilde{\mathbf x}_{n}
\end{array}
\right],\
\widetilde {\mathbf F}_e:=\left[
\begin{array}{l}
\tilde{\mathbf x}_{n+1}\\
\vdots\\
\tilde{\mathbf x}_{n+m}
\end{array}
\right].
\]
Let $(\mathbf{\widetilde w}^*, \widetilde b^*)$ be the optimal solution of (\ref{eq:svc-l2-linear-reduce}),
then the predicting label for testing data $\mathbf x_i$ is
\[
\hbox{sign}((\mathbf{\widetilde w}^*)^\top \tilde{\mathbf x}_{i} + \widetilde b^*), \ i\in \{n+1,n+2,\cdots, n+m\}.
\]
We call (\ref{eq:svc-l2-linear-reduce}) the reduced $L_2$-loss linear SVM.

\section{Error analysis on the approximation of kernel matrices}\label{sec-error}

In this section, we analyze the  difference between $\mathbf{w}$ obtained by training $L_2$-loss kernel SVM on $\mathbf X_r=\{x_1,x_2,\cdots,x_n\}$ with kernel matrix being $\mathbf{K}^{(k)}_{rr}$ and $\mathbf{\widetilde w}$ with $\mathbf{\widetilde K}_{rr}$. Here for convenience, we omit the bias term of SVMs.

Denote that $\mathbf W$ is a set of all the possible $\mathbf w,\mathbf{\widetilde w}$ obtained by solving two problems above. Let $\bm \psi(\cdot),\kappa(\cdot,\cdot)$ be the feature mapping function and the kernel function associated to $\mathbf{K}^{(k)}_{rr}$, and $\widetilde{\bm\psi}(\cdot),\widetilde{\kappa}(\cdot,\cdot)$ associated to $\mathbf{\widetilde K}_{rr}$. The hypothesis sets we consider are 
$$\mathbf H = \{\ h(\cdot)\ |\ \exists\ \mathbf w \in \mathbf W, \forall\ \mathbf x\in \mathbf X_r, h(\mathbf x) = \mathbf w^\top  \bm \psi(x)\}$$
  and 
$$\mathbf{\widetilde H} = \{\ \widetilde h(\cdot)\ | \ \exists\ \mathbf{\widetilde w} \in \mathbf W,\forall\ \mathbf x\in \mathbf X_r, \widetilde h(\mathbf x) = \mathbf{\widetilde w}^\top  \bm{\widetilde\psi}(x)\}.$$ 
We'll use the same notation in following passage and Appendix \ref{apd:proof_w}.

We need the following assumption which is also used in \citet[Chapter 6, Page 117]{mohri2018foundations}.

\begin{assumption} \label{assumption}
Assume that there exist $\rho \geq 0$,\ $G\geq 0$ such that \be\label{ass-1}\max\{\kappa(\mathbf x,\mathbf x),\ \widetilde \kappa(\mathbf x,\mathbf x)\} \leq \rho,\  \forall\ \mathbf x \in \mathbf X_r,\ee
and \be\label{ass-2}\max\{|h(\mathbf x)|,\ |\widetilde{h}(\mathbf x)|\}\leq G, \ \forall\ \mathbf x \in \mathbf X_r,\ h(\cdot)\in \mathbf{H}, \ \forall\ \widetilde h(\cdot)\in\mathbf{\widetilde H}.\ee
\end{assumption}

Our result is stated as follows, whose proof is in Appendix \ref{apd:proof_w}.
\begin{theorem}\label{thm:bound_w}
Under Assumption \ref{assumption},  we have
$$
\|\mathbf w - \mathbf{\widetilde w} \|^2 \leq 4C_0^2G(G+1)\rho^{\frac{1}{2}}\left[ke^{\frac{1}{4}}_f + \lambda_1tr(\mathbf A)+ke_2tr(\widetilde{\bm\Lambda}_{(k)}^{-1})\left(e_f^{\frac{1}{4}}+tr(\bm\Lambda_{(k)}^2)^{\frac{1}{4}}\right)\right],
$$
where $\mathbf A\in \mathbb R^{k\times k}$ is a diagonal matrix with entries
$$A_{ii} = max(\frac{1}{\widetilde{\Lambda}_{ii}}+\frac{1}{\Lambda_{ii}},\frac{3}{\widetilde{\Lambda}_{ii}}-\frac{1}{\Lambda_{ii}}),$$
$\bm \Lambda = diag(\lambda_1,\cdots, \lambda_n)$ and $\widetilde{\bm \Lambda} = diag(\widetilde{\lambda}_1, \cdots, \widetilde{\lambda}_k)$ are the exact and Nystr{\"o}m approximate eigenvalues (sorted in descending order) of the kernel matrix and
\[
\begin{array}{ll}
e_f &= (\sum_{i=k+1}^n \lambda_i^2)^{\frac{1}{2}}+\xi_f,\\
e_2 &= \lambda_{k+1}+\xi_2,
\end{array}
\]
in which $\xi_f$ and $\xi_2$ are known error bounds on the gaps between the Nystr{\"o}m low-rank approximation and the original kernel matrix with following definition
\[
\begin{array}{ll}
\xi_f&:=\|\mathbf K_{rr} - \mathbf{\widetilde K}_{rr}\|_F,\\
\xi_2&:=\|\mathbf K_{rr} - \mathbf{\widetilde K}_{rr}\|_2.
\end{array}
\]
\end{theorem}

From this theorem, we can see that $\|\mathbf w- \widetilde{\mathbf w}\|$ is bounded by the gap between the Nystr{\"o}m low-rank approximation and the original kernel matrix. Therefore the smaller the approximation error of kernel matrix, the smaller the $\|\mathbf w- \widetilde{\mathbf w}\|$ is, i.e., the more accurate the $\widetilde{\mathbf w}$ is. 

Let $(\mathbf w^\kappa, b^\kappa, \mathbf y_e^\kappa, \mathbf h_e^\kappa)=L_2(\mathbf X_r,\mathbf y_r,\mathbf X_e;\bm\psi)$ denote the solution $(\mathbf w^\kappa, b^\kappa)$ of training $L_2$-loss kernel SVM trained on data $(\mathbf X_r,\mathbf y_r)$ with kernel $\bm\psi$, and $\mathbf y_e^\kappa$, $\mathbf h_e^\kappa$ denote the predicting labels and values respectively, for test data $\mathbf X_e$, i.e., $(\mathbf h_e^\kappa)_i = \mathbf w^\kappa\bm\psi(\mathbf x_{n+i})+ b^\kappa$. Let $\mathcal I$ be the identity operator. We have similar notations for others. The content in this article has the relations as shown in Fig.~\ref{fig:structure}. 

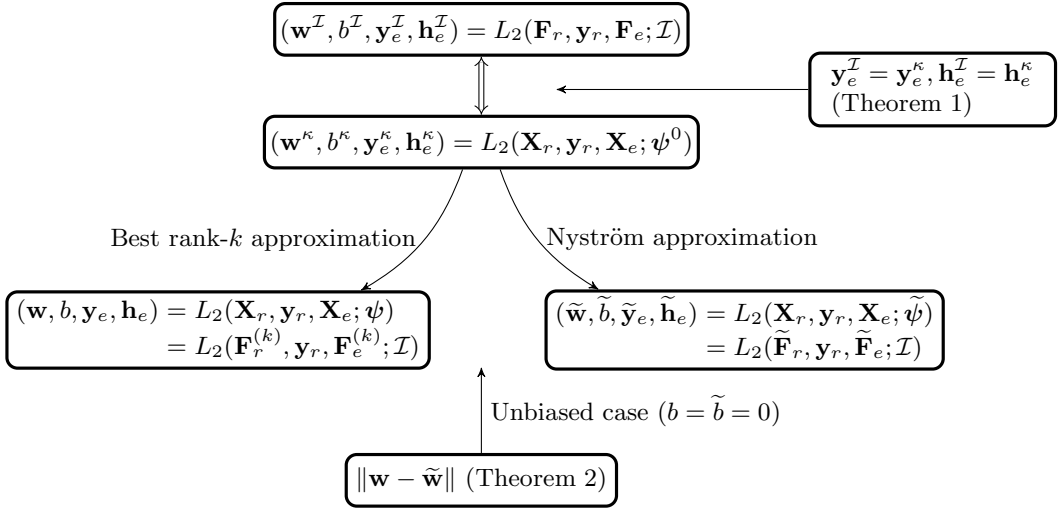
\begin{figure}
\begin{tikzpicture}[->,>=stealth']

 % Position of QUERY 
 % Use previously defined 'state' as layout (see above)
 % use tabular for content to get columns/rows
 % parbox to limit width of the listing
 \node[state] (L2LIN)
 {$(\mathbf w^{\mathcal I}, b^{\mathcal I}, \mathbf y_e^{\mathcal I}, \mathbf h_e^{\mathcal I})=L_2(\mathbf F_r,\mathbf y_r,\mathbf F_e; \mathcal I)$};

 \node[state,
  below of=L2LIN,
  yshift=-0.5cm,
  anchor=center] (L2KER) 
 {%
$(\mathbf w^\kappa, b^\kappa, \mathbf y_e^\kappa, \mathbf h_e^\kappa)=L_2(\mathbf X_r,\mathbf y_r,\mathbf X_e;\bm\psi^0)$
 };

 \node[state,
   right of=L2LIN,
   node distance=6cm,
   yshift=-0.8cm] (TH31)
 {
\begin{tabular}{l}
 $\mathbf y_e^{\mathcal I} = \mathbf y_e^\kappa, \mathbf h_e^{\mathcal I} = \mathbf h_e^\kappa$\\
 (Theorem \ref{thm-2})
 \end{tabular}
 };

 \node[state,
  below of=L2KER,
  xshift=-3.5cm,
  yshift=-1.5cm] (KERK) 
 {%
$\begin{array}{rl}
(\mathbf w, b, \mathbf y_e, \mathbf h_e)&=L_2(\mathbf X_r,\mathbf y_r,\mathbf X_e;\bm\psi)\\
 \hbox{}&=L_2(\mathbf F^{(k)}_r,\mathbf y_r,\mathbf F^{(k)}_e;\mathcal I)
\end{array}$
 };

 \node[state,
  right of=KERK,
  node distance=7cm] (KERN) 
 {%
$\begin{array}{rl}
(\widetilde{\mathbf w}, \widetilde{b}, \widetilde{\mathbf y}_e, \widetilde{\mathbf h}_e)&=L_2(\mathbf X_r,\mathbf y_r,\mathbf X_e;\widetilde{\bm\psi})\\
 \hbox{}&=L_2(\widetilde{\mathbf F}_r,\mathbf y_r,\widetilde{\mathbf F}_e;\mathcal I)
\end{array}$

 };

 \node[state,
  below of=KERK,
  xshift=3.5cm,
  yshift=-1cm] (TH41) 
 {%
$\begin{array}{c}
\|\mathbf w- \widetilde{\mathbf w}\| \ \hbox{(Theorem \ref{thm:bound_w})}
\end{array}$
 };

\draw[->] (TH31) -- (1,-0.8); 
\path
(L2LIN) edge[implies-implies,double equal sign distance] (L2KER)
(L2KER) edge[bend left=20]                node[anchor=right,left]{Best rank-$k$ approximation} (KERK)
(L2KER) edge[bend right=20]               node[anchor=left,right]{Nystr{\"o}m approximation} (KERN)
(TH41)  edge              node[anchor=left,right]{Unbiased case ($b=\widetilde{b}=0$)} (0,-4.5)
;

\end{tikzpicture}

\caption{Structure of our article}
\label{fig:structure}
\end{figure}

\paragraph{Remark:} \citet{zhang2012scaling} also do similar work as in Fig.~\ref{fig:structure}, but for $L_1$-loss kernel SVM and corresponding $L_1$-loss linear SVM. Another difference of our work from theirs is that we employ the latest highly efficient semismooth Newton's method to solve the reduced linear SVM.

\section{Semismooth Newton's method for the reduced $L_2$-loss linear SVM}\label{sec-subproblem}
Next, we will apply semismooth Newton's method \citep{yin2019semismooth} to solve the reduced $L_2$-loss linear SVM (\ref{eq:svc-l2-linear-reduce}), which is equivalent to the following unconstrained problem
\begin{equation}\label{eq-svc-l2-unc}
\min_{\mathbf w\in\mathbb{R}^p, b \in \mathbb{R}} \frac{1}{2}\|\mathbf w\|^2 + C\sum_{i=1}^n \max(1-y_i(\mathbf w^\top \tilde{\mathbf x}_i+b) , 0)^2.
\end{equation}
Due to the fact that the bias term $b$ hardly affect the numerical performance as shown in \citep[Section 4.5]{DCDSVR}, we omit the bias term, and solve the following unbiased model (by setting $\tilde {\mathbf x}_i \leftarrow[\tilde {\mathbf x}_i, 1], \ \tilde {\mathbf w}_i \leftarrow[\tilde {\mathbf w}_i, 1]$)
\begin{equation}\label{eq-svc-l2-unc-unbiased}
\min_{\mathbf w\in\mathbb{R}^p} \frac{1}{2}\|\mathbf w\|^2 + C\sum_{i=1}^n \max(1-y_i\mathbf w^\top \tilde{\mathbf x}_i, 0)^2:=f(\mathbf w).
\end{equation}
At iteration $j$, let $I_j:\ =\{i:\ 1-y_i\tilde {\mathbf x}_i^\top  \mathbf w^j >0\}$ and denote
\[
\hat\partial ^2 f(\mathbf w^j) = \{  I + 2C\sum_{i \in I_j}^l\tilde{ \mathbf x}_i\tilde{\mathbf x}_i^\top  \}.
\]

Details of the semismooth Newton's method are given in Alg. \ref{alg-seminewton}.
%By exploring the sparse structure of the model (\ref{eq-svc-l2-unc}), this algorithm reduce the computational complexity, meanwhile keeping the quadratic convergence rate. The algorithm is as below. \\
\begin{algorithm}\label{alg-seminewton}
\caption{A globalized semismooth Newton's method for (\ref{eq-svc-l2-unc-unbiased})}
\LinesNumbered
Given $j=0$. Choose $ w^0, \sigma\in(0,1), \rho \in(0,1), \delta >0$ and $\eta_0 >0, \eta_1>0$\;
Calculate $\nabla f(\mathbf w^j)$. If $\|\nabla f(\mathbf w^j)\| \leq \delta$, stop. Otherwise go to step 3 \;
Select an element $\mathbf V^j \in \hat\partial^2f(\mathbf w^j)$ and apply Conjugate Gradient (CG) method\citep{hestenes1952methods} to find an approximate solution $\mathbf d^j$ by
\begin{equation*}
\mathbf V^j\mathbf d^j+\nabla f(\mathbf  w^j) =0
\end{equation*}
such that
\begin{equation*}
\|\mathbf V^j\mathbf d^j+\nabla f(\mathbf  w^j)\| \leq \mu_j\|\nabla f(\mathbf  w^j)\|
\end{equation*}
where $\mu_j = min(\eta_0, \eta_1\|\nabla f(\mathbf w^j)\|)$\;
Do line search to find the smallest positive integer $m_j$ such that the following holds
$$f(\mathbf  w^j+\rho^m\mathbf d^j)\leq f(\mathbf  w^j) + \sigma \rho^m\nabla f(\mathbf  w^j)^\top\mathbf  d^j.$$
Let $\alpha_j = \rho^{m_j}$\;
Let $ \mathbf w^{j+1} =  \mathbf w^j +\alpha^j\mathbf d^j, j = j+1$. Go to step 2\;
\end{algorithm}

The advantage of this semismooth Newton's method is that it enjoys global convergence and quadratic convergence rate, as we show in the following theorem. %Moreover, the computational cost in each iteration can be significantly reduced to $O(n|I_j|)$  by exploring the sparse structure in $\hat\partial^2 f(\mathbf w)$ \citep[Section 4.3]{yin2019semismooth}, which is different from the traditional heaving computational cost of second-order methods.

\paragraph{Remark} As analyzed by \citet{yin2019semismooth}, the main computational cost in each iterations of semismooth Newton's method lies in Step 3 of Alg.~\ref{alg-seminewton}, which is to calculate $\mathbf V^j\mathbf h$, for any $\mathbf h \in \mathbb{R}^p$. By exploring the sparse structure of the optimal solution of (\ref{eq-svc-l2-unc-unbiased}), the computational cost of computing $\mathbf V^j\mathbf h$ can be reduced to O$(n|I_j|)$, where $|I_j|$ is the number of elements in the set $I_j$ and $|I_j|\ll n$.

\begin{theorem}\citep[Theorem 1]{yin2019semismooth} \label{thm-1} Let $\mathbf w^*$
be  a solution of  (\ref{eq-svc-l2-unc}).  Then every sequence generated by  (\ref{alg-seminewton})  is quadratically convergent to $\mathbf w^*$. 
\end{theorem}

Now we summarize our Semismooth-Newton's-method-based Linearization and Approximation approach (LASN) as follows. Firstly, we choose $k$ landmark points by k-means clustering algorithm, then use Nystr{\"o}m approximation to get $\mathbf{\widetilde K}_{ll}$. Then we get $\mathbf{\widetilde{F}}_r$, and train linear SVM by semismooth Newton's method. The details of our approach is given in Alg.~\ref{alg-2} and the predicting process are given in Alg.~\ref{alg-2-p}.

In Alg.~\ref{alg-2}, the first four steps take O$(nkp+k^3+nk^2)$ operations, where $k$ is usually between $n/10$ and $n/100$.

\begin{algorithm}\label{alg-2}
\caption{LASN Training stage}
\LinesNumbered
%%%%error
\KwIn{training data $\mathbf{X}_r$, training label $\mathbf{y}_r$, cost parameter $C$, landmark set size $k$}
\KwOut{weight vector $\hat{\mathbf{w}}$, mapping matrix $\mathbf{M}$}
Choose $k$ landmark points $L$ by efficient k-means method and then compute $\mathbf{K}_{ll}$ and $\mathbf{K}_{rl}$.\;
Compute spectral decomposition of $\mathbf{K}_{ll}$ to get $\mathbf{V}, \bm{\Lambda}$ such that $\mathbf{K}_{ll} = \mathbf{V}\bm{\Lambda}\mathbf{V}^{-1}$\;
Compute $\bm{\Lambda}^{-\frac{1}{2}}$ and then the mapping matrix $\mathbf{M}=\mathbf{V}\bm{\Lambda}^{-\frac{1}{2}}$ \;
Compute $\mathbf{K}_{rl}\mathbf{M}$\;
Train $L_2$-loss linear SVM on $\mathbf{K}_{rl}\mathbf{M}$ by semismooth Newton's method\;
\end{algorithm}

\begin{algorithm}\label{alg-2-p}
\caption{LASN Predicting stage}
\LinesNumbered
\KwIn{testing data $\mathbf{X}_e$, weight vector $\hat{\mathbf{w}}$, mapping matrix $\mathbf{M}$}
\KwOut{predicting label $\hat{\mathbf{y}}_e$ }
Compute $\mathbf{K}_{er}$\;
Predict by $\hat{\mathbf{y}}_e=\hbox{sign}(\mathbf{K}_{er}\mathbf{M}\hat{\mathbf{w}}$)\;
\end{algorithm}

\section{Numerical result}\label{sec-numerical}
In this section, we conduct extensive numerical test to verify the efficiency of our method. It is divided into three parts. In the first part,  we analyze how to choose landmark set size $k$ and cost parameter $C$ for the proposed method.  Then we compare numerical results of low-rank linearized method with different linear SVMs' solvers. Finally, we compare  the performance of our algorithm and the solver in LIBSVM.

All experiments are tested in Matlab R2018b in Windows 10 on a Microsoft Surface Pro 4 with an Intel(R) Core(TM) i5-6300U CPU at 2.40 GHz, 2.50 GHz and of 8 GB RAM. Throughout the experiment, we use the Gaussian kernel $\kappa(\mathbf x, \mathbf y)=\exp(-\|\mathbf x-\mathbf y\|^2/\gamma)$ where $\gamma$ is chosen as the average squared distance between data points of each dataset \citep{kumar2009sampling}.

We use standard real datasets available at LIBSVM site\footnote{\url{https://www.csie.ntu.edu.tw/~cjlin/libsvmtools/datasets/}}. Due to the need of computing kernel matrix and sampling with k-means method, our algorithm is hard to tackle with dataset with millions of instance or each instance having millions of features on limited computing resources (eg. PC) and lose it is efficiency. Therefore we screen out those datasets with $p*n>10^{10}$. For datasets without explicitly splitting into training set and testing set, we use the first 60\% data points as training set and the other 40\% as testing set.

\paragraph{Implementations} In Step 1 of Alg. \ref{alg-2}, we adopted a fast approximate k-means sampling method using only a few iterations, which does not necessarily converge. Then we use the center obtained from the k-means method as landmark points. In the fast k-means sampling procedure, if one particular dataset has more than 20000 data points, for efficiency we only use the first 20000 data points to choose landmark points. In Step 4 of Alg. \ref{alg-2}, let $\bm{\Lambda} = \hbox{diag}\{a_1, a_2, \cdots, a_k\}$, we set $a_i = 0$ if $a_i <10^{-6}$ then compute the inverse of $\bm{\Lambda}$. We improve the efficiency of this algorithm by computing the distance between data points in advance and using it in following multiple steps.

\subsection{Choosing parameters}
In this part, we analyze the impact of different choices of landmark set size $k$ and the cost parameter $C$. For each dataset we randomly choose 80\% data points of original training set for training and the other 20\% of training set for cross-validation, so we are going to report two accuracies: accuracy of predicting on the cross-validation set (CV accuracy) and accuracy of predicting on the testing set (testing accuracy).

\subsubsection{Landmark set size k}
From the theoretical point of view, the larger $k$ is, the more accurate the approximation of the kernel matrix is. However, recall that $n$ is the number of training data points, $k$ must satisfy $k\le n$. On the other hand, $k$ is also the dimension of the transformed linear data points, which means that the SVM is going to be trained on data with dimensions of $n\times k$. Consequently, for efficiency and saving computing resources, $k$ shouldn't be too large.

We test on the datasets reported in Table \ref{tab:addlabel}.
According to size of these datasets, we choose $k=\{50,100,200,500,1000\}$ with corresponding $log(k)=\{1.7, 2.0, 2.3, 2.7, 3.0\}$. Here we set cost parameter $C=10$. Recall the notation for number of features is $p$ and the one for number of instance is $n$. The datasets can be divided into three groups: Large Datasets (LD) with $p*n>10^8$, Medium Datasets (MD) with $10^6<p*n<10^8$ and Small Datasets (SD) with $p*n<10^6$.

% Information of datasets are showed below.\\
% \par \noindent
\renewcommand{\arraystretch}{1.4}
\begin{table}[htbp]
  \centering
  \caption{Groups of datasets}
  \resizebox{\textwidth}{!}
  {
    \begin{tabular}{lcccc}
    \multicolumn{1}{l|}{\textbf{Dataset name}} & \multicolumn{1}{c|}{\textbf{number of instance $n$}} & \multicolumn{1}{c|}{\textbf{number of features $p$}} & \multicolumn{1}{c|}{\textbf{kernel parameter $\gamma$}} & \multicolumn{1}{c}{\textbf{Group}}\\
  skin\_nonskin  & 42244 & 3       & 4843.688 & SD\\
  a2a            & 1812  & 123     & 7.6484 & SD\\
  ijcnn          & 39992 & 22      & 1.2974 & SD\\\hline
  cod-rna        & 47628 & 8       & 36592.4 & MD\\
  a9a            & 26048 & 123     & 7.6723 & MD\\
  gisette\_scale & 4800  & 5000    & 1287.527 & MD\\\hline
  rcv1\_binary   & 16193 & 47236   & 0.98094 & LD\\
  news20\_binary & 9597  & 1355191 & 0.94151 & LD\\
  real-sim       & 34708 & 20958   & 0.98791 & LD\\
    \end{tabular}%
    }
  \label{tab:addlabel}%
\end{table}%
% large datasets:'rcv1_binary', 'news20_binary', 'real-sim'
% middle datasets: 'cod-rna', 'a9a', 'gisette_scale'
% small datasets: 'skin_nonskin', 'a2a', 'ijcnn', 'phishing'

\begin{itemize}
\item Large datasets
\begin{figure}[ht]
\centering
\includegraphics[scale=0.65]{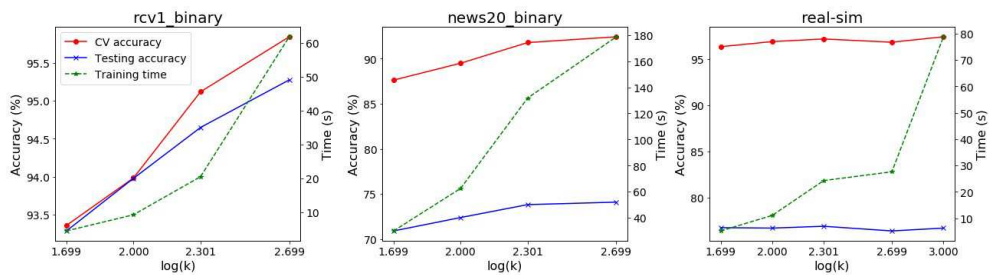}
\caption{Results of LASN on large datasets with different k's}
\label{fig:ld_k}
\end{figure} 

As showed in Fig.~\ref{fig:ld_k}, for this group of datasets, larger $k$ usually means higher predicting accuracy but longer training time too. There is a tradeoff between training time and predicting accuracy. Letting $k = \sqrt{n}$ could be a good choice.

\item Medium datasets

\begin{figure}[ht]
\centering
\includegraphics[scale=0.65]{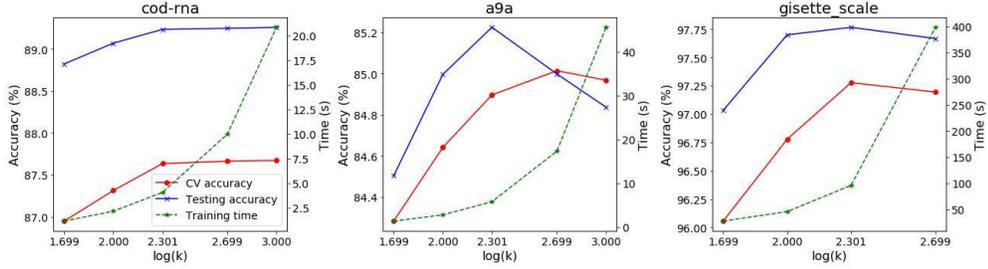}
\caption{Results of LASN on medium datasets with different k's}
\label{fig:nd_k}
\end{figure}

In Fig.~\ref{fig:nd_k}, for this group of datasets, letting $k = \sqrt{n}$ again could be a good choice, since larger $k$ won't significantly improve predicting accuracy and it makes the training process consume more time.\\

\item Small datasets

\begin{figure}[ht]
\centering
\includegraphics[scale=0.65]{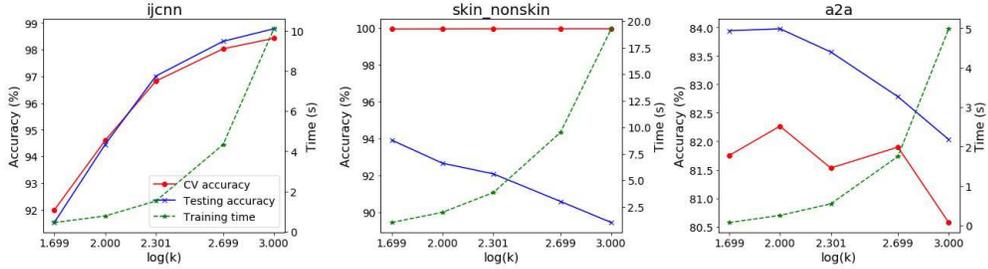}
\caption{Results of LASN on small datasets with different k's}
\label{fig:sd_k}
\end{figure}

From Fig.~\ref{fig:sd_k}, for small datasets, using $k = \sqrt{n}$ could be a good choice too. However, increasing $k$ can bring different outcomes, which can be related to the nature of each unique dataset. Luckily, the training time on this group of datasets are short no matter how large $k$ is.  We can try several $k$'s around $\sqrt{n}$ to find out the best choice.
\end{itemize}%

In all the figures above, we can see that generally the training time of our algorithm increase linearly with the number of landmark points. It means that our algorithm is scalable through adjusting the parameter $k$.

\subsubsection{Cost parameter C} 

According to the results above, we have chosen the landmark point size $k$ for each dataset. Now we are going to explore proper choice of the  cost parameter $C$ given the chosen $k$ of each dataset. Fig.~\ref{fig:stable_c} and Fig.~\ref{fig:nonstable_c} displays the testing accuracy, CV accuracy and the time of training on several datasets with fixed $k$ and different $C$'s ranging from $10^{-4}$ to $10^{4}$.

In Fig.~\ref{fig:stable_c}, the training time of our algorithm increase as $C$ becomes larger in an accelerated speed and increase significantly from $C=10$ to $C=100$. For each dataset, accuracy generally increase as $C$ becomes larger and comes to be stable when $C=10$. In Fig.~\ref{fig:nonstable_c}, we have similar results to that in  Fig.~\ref{fig:stable_c}. However, training time is unstable as $C$ increases and the range of training time is small for all datasets except real\_sim.

To conclude, $C$ being $10$ is usually a good choice, with which our algorithm usually can get a good accuracy of predicting without using too much time.

\begin{figure}[ht]
\centering
\includegraphics[scale=0.7]{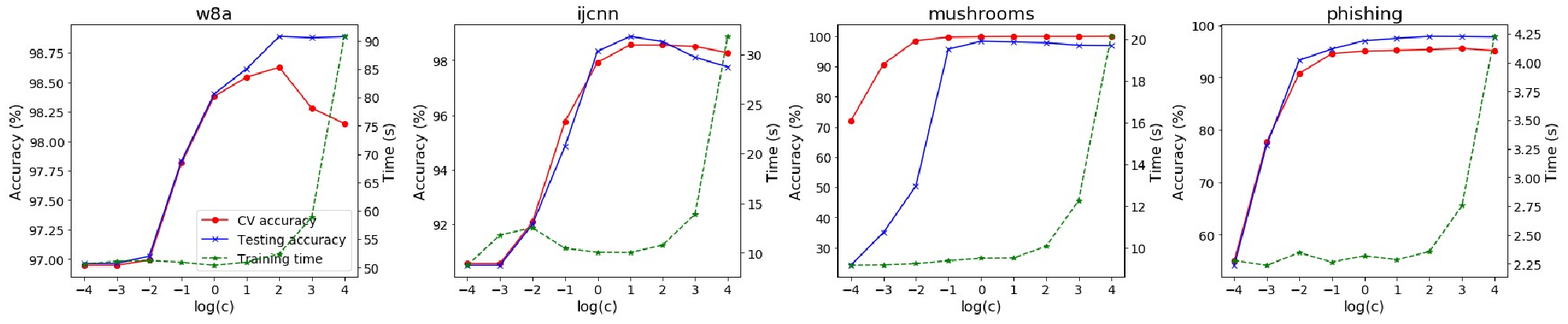}
\caption{Results of LASN with different C's, part \RNum{1}}
\label{fig:stable_c}
\end{figure}

\begin{figure}[ht]
\centering
\includegraphics[scale=0.7]{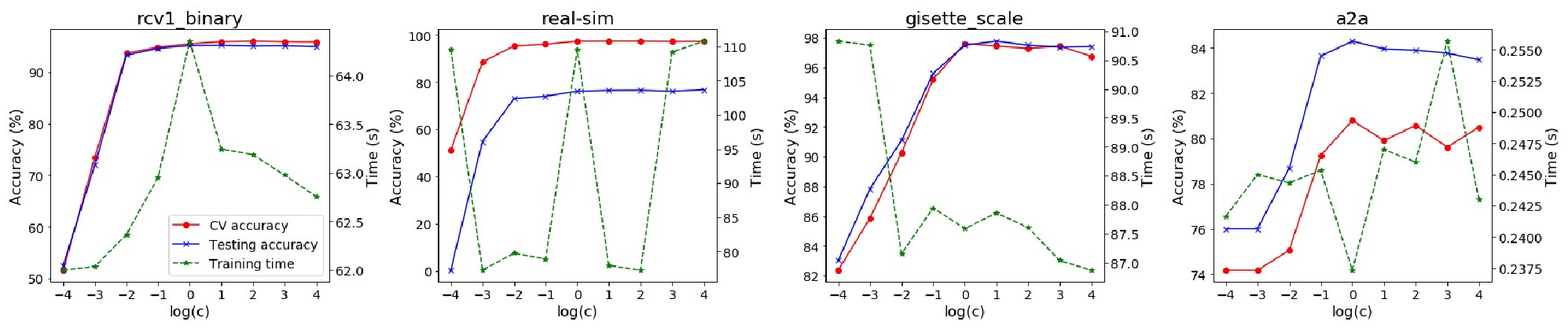}
\caption{Results of LASN with different C's, part \RNum{2}}
\label{fig:nonstable_c}
\end{figure}

\subsection{Numerical comparisons between different linear solvers}\label{subsec-lin}
One important step in our algorithm is Step 5, which is to solve the reduced linear SVM with semismooth Newton's method.  In this part, as comparisons, we use two solvers in liblinear, which are the dual coordinate descent method (DCD) \citep{hsieh2008dual} and the trust region Newton method (TRON) \citep{lin2008trust}, and compare the numerical results with our algorithm. For the three resulting algorithms, denoted as LA-SN (Alg. \ref{alg-2} with semismooth Newton's method as subsolver), LA-DCD (Alg. \ref{alg-2} with DCD as subsolver) and LA-TRON (Alg. \ref{alg-2} with TRON as subsolver), we'll use the same cost parameter $C=10$ the number of landmark points $k$ for each dataset. For each dataset, we repeat each algorithm for ten times, and report the mean of the time used for solving the reduced linear SVM (denoted by $t_{linear}$) and predicting accuracies.

Detailed information of datasets used in Subsection~\ref{subsec-lin} and Subsection~\ref{subsec-lib} is given in Table \ref{tab:datasets}.

\begin{table}[htbp]
  \caption{Information of datasets used in Subsection~\ref{subsec-lin} and Subsection~\ref{subsec-lib}}
    \begin{tabular}{c|cc}
    Dataset name & number of instance $n$ & number of features $p$ \\
    \noalign{\smallskip}\hline\noalign{\smallskip}
    ijcnn & 49990 & 23 \\
    w1a   & 2477  & 300 \\
    w2a   & 3470  & 300 \\
    w3a   & 2990  & 300 \\
    w4a   & 3618  & 300 \\
    w5a   & 9888  & 300 \\
    w6a   & 17188 & 300 \\
    w7a   & 13961 & 300 \\
    w8a   & 43735 & 300 \\
    phishing & 2962  & 67 \\
    mushrooms & 4874  & 112 \\
    real-sim & 43385 & 20958 \\
    skin\_nonskin & 52806 & 3 \\
    cod-rna & 59535 & 8 \\
    madelon & 323   & 500 \\
    liver-disorders & 145   & 5 \\
    a1a   & 1605  & 123 \\
    a2a   & 2265  & 123 \\
    a3a   & 3185  & 123 \\
    a4a   & 4781  & 123 \\
    a5a   & 6414  & 123 \\
    a6a   & 6414  & 123 \\
    a7a   & 6414  & 123 \\
    a8a   & 3318  & 123 \\
    a9a   & 32561 & 123 \\
    rcv1\_binary & 20242 & 47236 \\
    news20\_binary & 11997 & 1355191 \\
    gisette\_scale & 6000  & 5000 \\
    \noalign{\smallskip}\hline
    \end{tabular}%
  \label{tab:datasets}%
\end{table}%

\begin{table}
  \caption{Numerical comparisons between different linear solvers}
  %\centering
  \begin{threeparttable}
    \begin{tabular}{l|cc|cc|cc}
          \multicolumn{1}{c}{} & \multicolumn{2}{c}{\textbf{LA-SN}\tnote{1}} & \multicolumn{2}{c}{\textbf{LA-DCD}\tnote{1}} & \multicolumn{2}{c}{\textbf{LA-TRON}\tnote{1}} \\
    Dataset & $t_{linear}$(s) & accuracy(\%) & $t_{linear}$(s) & accuracy(\%) & $t_{linear}$(s) & accuracy(\%) \\
    \noalign{\smallskip}\hline\noalign{\smallskip}
    ijcnn           & \textbf{11.6998} & 98.9363          & 15.6948        & \textbf{98.94}   & 14.2318 & \textbf{98.94} \\
    w1a             & \textbf{0.0454}  & 97.5478          & 0.1234         & \textbf{97.5491} & 0.125   & \textbf{97.5491} \\
    w2a             & \textbf{0.079}   & 97.6516          & 0.188          & \textbf{97.6521} & 0.192   & \textbf{97.6521} \\
    w3a             & \textbf{0.077}   & \textbf{97.7523} & 0.1664         & 97.7501          & 0.1738  & 97.7501 \\
    w4a             & \textbf{0.0938}  & \textbf{97.8444} & 0.2222         & 97.8388          & 0.2254  & 97.8388 \\
    w5a             & \textbf{1.4526}  & \textbf{98.2755} & 2.2548         & 98.2755          & 2.2618  & 98.2755 \\
    w6a             & \textbf{2.3096}  & \textbf{98.5332} & 4.1908         & \textbf{98.5332} & 4.216   & \textbf{98.5332} \\
    w7a             & \textbf{2.1148}  & \textbf{98.576}  & 3.4694         & 98.5752          & 3.4396  & 98.5752 \\
    w8a             & \textbf{5.9344}  & \textbf{98.7947} & 10.9314        & 98.7934          & 10.9452 & 98.7934 \\
    phishing        & \textbf{0.0236}  & \textbf{95.4453} & 0.0686         & 95.3846          & 0.074   & 95.3846 \\
    mushrooms       & \textbf{0.084}   & 95.343           & 0.121          & \textbf{95.3553} & 0.1274  & \textbf{95.3553} \\
    real-sim        & \textbf{8.2336}  & 81.4126          & 13.7686        & \textbf{81.4361} & 13.7766 & \textbf{81.4361} \\
    skin\_nonskin   & \textbf{0.1554}  & 94.2093          & 0.6256         & \textbf{94.1946} & 0.603   & \textbf{94.1946} \\
    cod-rna         & \textbf{0.2664}  & \textbf{89.2636} & 5.3674         & 89.2601          & 5.3562  & 89.2601 \\
    madelon         & 0.0178           & \textbf{62}      & \textbf{0.009} & 61.6667          & 0.0118  & 61.6667 \\
    liver-disorders & \textbf{0.006}   & \textbf{58.8}    & 0.0082         & 58.5             & 0.0096  & 58.5 \\
    a1a             & \textbf{0.0362}  & \textbf{83.5883} & 0.1006         & 83.5793          & 0.1034  & 83.5793 \\
    a2a             & \textbf{0.0566}  & \textbf{83.648}  & 0.1666         & 83.6044          & 0.1632  & 83.6044 \\
    a3a             & \textbf{0.1006}  & \textbf{84.04}   & 0.2388         & 84.0298          & 0.2412  & 84.0298 \\
    a4a             & \textbf{0.157}   & 84.3549          & 0.3688         & \textbf{84.3585} & 0.377   & \textbf{84.3585} \\
    a5a             & \textbf{0.204}   & \textbf{84.4877} & 0.512          & 84.4862          & 0.509   & 84.4862 \\
    a6a             & \textbf{0.1938}  & \textbf{84.4637} & 0.5084         & 84.4581          & 0.5162  & 84.4581 \\
    a7a             & \textbf{0.2064}  & \textbf{84.4991} & 0.5106         & 84.4858          & 0.5122  & 84.4858 \\
    a8a             & \textbf{0.106}   & \textbf{84.2636} & 0.2458         & 84.2534          & 0.2526  & 84.2534 \\
    a9a             & \textbf{0.6758}  & \textbf{85.0709} & 3.5074         & 85.0562          & 3.518   & 85.0562 \\
    rcv1\_binary    & \textbf{2.421}   & \textbf{95.7131} & 4.0992         & \textbf{95.7131} & 4.0766  & \textbf{95.7131} \\
    news20\_binary  & \textbf{2.1738}  & \textbf{73.31}   & 3.49           & 73.3075          & 3.5296  & 73.3075 \\
    gisette\_scale  & \textbf{0.0564}  & \textbf{97.56}   & 0.141          & \textbf{97.56}   & 0.1426  & \textbf{97.56} \\
    \noalign{\smallskip}\hline
    \end{tabular}%
   % \begin{tablenotes}
%       \footnotesize
%       \item[1] SN:Semismooth Newton's method, DCD:Dual Coordinate Descent method, TRON:Trust Region Newton method.
%       \item[2] $t_{linear}$ stands for time used to solve the linearized kernel SVM.
%     \end{tablenotes}
   \end{threeparttable}
  \label{tab:linearsolver}%
\end{table}%

In Table \ref{tab:linearsolver}, we can see that  the three algorithms have similar performance in predicting accuracies, but the semismooth Newton's method costs much less time than DCD and TRON in training the reduced linear SVM. It verifies that the semismooth Newton's method is a good choice for the linearization and approximation approach.

\subsection{Numerical comparisons with LIBSVM}\label{subsec-lib}
In this part, we compare our algorithm with LIBSVM \citep{CC01a}. For both algorithms, we'll use the same cost parameter $C=10$. Since the results of LASN are non-deterministic, for each dataset we repeat LASN for ten times and then report mean and standard deviation of training time and predicting accuracies on testing set.

\begin{table}
  %\centering
  \caption{Numerical comparisons between LASN and LIBSVM}
    \begin{tabular}{l|cc|ccc}
    \multicolumn{1}{c}{} & \multicolumn{2}{c}{\textbf{LIBSVM($L_2$-loss)}} & \multicolumn{3}{c}{\textbf{LASN}} \\
    Dataset & $t$(s) & accuracy(\%) & $t$(s) & accuracy(\%) & k \\
    \noalign{\smallskip}\hline\noalign{\smallskip}
    ijcnn           & 127.46         & 98.98           & \textbf{72.98(1.18)  }  & \textbf{99.07 (0.02)} & 3000 \\
    w1a             & 0.86           & 97.24           & \textbf{0.80 (0.21) }   & \textbf{97.55 (0.11)} & 200 \\
    w2a             & 1.70           & 97.32           & \textbf{0.92 (0.02) }   & \textbf{97.62 (0.10)} & 200 \\
    w3a             & 1.25           & 97.34           & \textbf{0.77 (0.01) }   & \textbf{97.82 (0.12)} & 200 \\
    w4a             & 1.73           & 97.41           & \textbf{0.92 (0.03) }   & \textbf{97.88 (0.07)} & 200 \\
    w5a             & 43.58          & 97.46           & \textbf{2.33 (0.03) }   & \textbf{97.70 (0.08)} & 200 \\
    w6a             & 136.57         & 97.60           & \textbf{4.09 (0.54) }   & \textbf{97.94 (0.11)} & 200 \\
    w7a             & 90.36          & 97.66           & \textbf{\color{red}4.33 (0.83) }   & \textbf{\color{red}98.04 (0.08)} & 200 \\
    w8a             & 819.98         & \textbf{99.45 } & \textbf{\color{red}7.58 (1.65) }   & 98.11         (0.08)  & 200 \\
    phishing        & \textbf{0.67 } & \textbf{98.08 } & 0.96         (0.26)     & 96.99         (0.16)  & 200 \\
    mushrooms       & 4.37           & 38.39           & \textbf{1.71 (0.37) }   & \textbf{95.66 (2.81)} & 200 \\
    real-sim        & 921.20         & 76.54           & \textbf{\color{red}119.61(34.83) } & \textbf{\color{red}81.03 (1.12)} & 1000 \\
    skin\_nonskin   & 119.12         & 32.54           & \textbf{20.08(0.12)  }  & \textbf{93.57 (0.24)} & 1000 \\
    cod-rna         & 344.82         & \textbf{92.35 } & \textbf{21.29(0.12)  }  & 89.65         (0.03)  & 1000 \\
    madelon         & \textbf{0.11 } & 50.00           & 8.36         (0.08)     & \textbf{62.00 (0.00)} & 323 \\
    liver-disorders & \textbf{0.00 } & 50.00           & 0.03         (0.00)     & \textbf{58.50 (0.00)} & 145 \\
    a1a             & 0.51           & 76.50           & \textbf{0.49 (0.02) }   & \textbf{83.61 (0.14)} & 200 \\
    a2a             & 1.13           & 76.89           & \textbf{0.71 (0.13) }   & \textbf{83.56 (0.11)} & 200 \\
    a3a             & 1.88           & 77.26           & \textbf{1.04 (0.28) }   & \textbf{84.00 (0.08)} & 200 \\
    a4a             & 3.90           & 78.05           & \textbf{1.34 (0.03) }   & \textbf{84.34 (0.15)} & 200 \\
    a5a             & 18.05          & 78.33           & \textbf{2.17 (0.56) }   & \textbf{84.47 (0.08)} & 200 \\
    a6a             & 18.08          & 78.19           & \textbf{1.71 (0.01) }   & \textbf{84.40 (0.10)} & 200 \\
    a7a             & 18.13          & 78.51           & \textbf{1.71 (0.02) }   & \textbf{84.54 (0.10)} & 200 \\
    a8a             & 1.85           & 77.66           & \textbf{1.34 (0.31) }   & \textbf{84.22 (0.17)} & 200 \\
    a9a             & 679.82         & 80.65           & \textbf{\color{red}7.44 (1.44) }   & \textbf{\color{red}85.03 (0.07)} & 200 \\
  rcv1\_binary    & 573.22    &\textbf{96.60}    &\textbf{173.02(0.54)}             &95.74(0.03) &1000\\
  news20\_binary  & 637.20    &\textbf{74.90}          &\textbf{444.03(4.25)}      &73.52(0.10) &1000\\
  gisette\_scale  & 132.50    &\textbf{98.00}          &\textbf{54.43(0.85)}        &97.66(0.18) &100\\
    \noalign{\smallskip}\hline
    \end{tabular}%
  \label{tab:libsvm}%
\end{table}%

As can be seen in table \ref{tab:libsvm}, with proper choice of $k$, LASN outperforms LIBSVM in terms of both training time and predicting accuracy on majority of the datasets (marked in bold). Especially for w8a, a9a and real-sim datasets (marked in red), we can clearly see our algorithm has great advancement in training speed while keeping good predicting accuracy. It should be noticed that $k$ can be chosen in a different way according to the user's need, our algorithm is scalable so that it can achieve greater training speed at the cost of predicting accuracy.

\section{Conclusions}\label{sec-conclusions}
In this paper, we proposed a new approach to deal with the kernel SVMs. After linearizing the kernel matrix and approximating it by Nystr{\"o}m methods, we solve the reduced linear SVM by the highly efficient semismooth Newton's method. We also provide theoretical guarantee for the new approach. Extensive numerical results demonstrate the efficiency of the proposed approach as well as the improvement over the existing state-of-the-art methods.

\newpage
\section{Declarations}
This is the section for declarations.

\noindent
{\bf Funding}
Dr. Li's research was supported by the National Natural Science Foundation of China (No.11671036).

\noindent
{\bf Conflicts of interest/Competing interests}
Not applicable.

\noindent
{\bf Availability of data and material}
All the datasets we use are available at LIBSVM site\footnote{\url{https://www.csie.ntu.edu.tw/~cjlin/libsvmtools/datasets/}}.

\noindent
{\bf Code availability}
We use custom code written by ourself, LIBLINEAR available at LIBLINEAR site\footnote{\url{https://www.csie.ntu.edu.tw/~cjlin/liblinear/}} and LIBSVM available at LIBSVM site\footnote{\url{https://www.csie.ntu.edu.tw/~cjlin/libsvm/}}.

\appendix
\appendixpage
\addappheadtotoc

\section{Proof of Theorem \ref{thm-2}}\label{apd:proof_eqv}
\begin{proof}
Recall the kernel $L_2$-loss SVM model (\ref{eq:l2-svc-kernel}).
%\be
%\label{model:l2kernel}
%\min_{\mathbf w, b}\frac{1}{2}\|\mathbf w\|^2+C\sum_{i=1}^n max(1-y_i(\mathbf w^\top \bm\psi(x_i)+b), 0)^2.
%\ee
Let $(\mathbf w^*, b^*)$ be the optimal solution of (\ref{eq:l2-svc-kernel}), then predicting label for any $\mathbf x\in \mathbf X_e$ is given by (\ref{pre-kernel}).
%$$
%sign((\mathbf w^*)^\top \bm\psi(\mathbf x)+b^*).
%$$
Denote
$$\mathbf{F} = \left[
\begin{array}{c}
\mathbf{F}_r\\
\mathbf{F}_e\\
\end{array}
\right]=\left[
\begin{array}{c}
\hat{\mathbf x}_1^\top \\
\dots\\
\hat{\mathbf x}_{n+m}^\top
\end{array}\right].$$
The $L_2$-loss SVM model to train $(\mathbf F_r,\mathbf y_r)$ is
\be
\label{model:l2linear}
\min_{\mathbf w, b}\frac{1}{2}\|\mathbf w\|^2+C\sum_{i=1}^n max(1-y_i(\mathbf w^\top \hat{\mathbf x}_i+b), 0)^2.
\ee
Let $(\widehat{\mathbf w}^*, \hat{b}^*)$ be the optimal solution of (\ref{model:l2linear}), then the predicting label for any $\hat{\mathbf x}\in \mathbf F_e$ is
$$
sign((\widehat{\mathbf w}^*)^\top \hat{\mathbf x}+b^*).
$$

To prove our theorem, we only to show that the dual problem of (\ref{eq:l2-svc-kernel}) and the dual problem of (\ref{model:l2linear}) are equivalent and so are their predicting labels.

Let $\mathbf A^\top  := [-y_1\bm\psi(\mathbf x_1), \cdots, -y_n\bm\psi(\mathbf x_n)]$ and $\mathbf e\in \mathbb{R}^n$ be a vector of all ones. We rewrite (\ref{eq:l2-svc-kernel}) equivalently as
\be
\label{model:l2kernelv}
\renewcommand\arraystretch{1.2}
\begin{array}{ll}
\min\limits_{\mathbf w, \bm \xi, b}&\frac{1}{2}\|\mathbf w\|^2+C\|\bm\xi\|^2\\
\hbox{s.t.} &\bm \xi \geq \mathbf{0},\\
\hbox{ } &\bm\xi \geq \mathbf{Aw}-b\mathbf{y+e}.
\end{array}
\ee
%where $\mathbf A^\top  := [-y_1\bm\psi(\mathbf x_1), \cdots, -y_n\bm\psi(\mathbf x_n)]$ and $\mathbf e\in \mathbb{R}^n$ is a vector of all ones.
The corresponding Lagrange function is
$$
L(\mathbf{w,\bm \xi,b;\bm \lambda,\bm \mu}) = \frac{1}{2}\|\mathbf w\|^2 +C \| \bm\xi\|^2- \langle\bm\xi, \bm\mu\rangle - \langle\bm\lambda, \bm\xi - \mathbf{Aw+by-e}\rangle .
$$
The KKT condition of problem (\ref{model:l2kernelv}) is
\be
\label{KKT}
\left\{
\renewcommand\arraystretch{1.5}
\begin{array}{l}
\nabla_\mathbf{w}L(\mathbf{w,\bm \xi,}b;\bm \lambda,\bm \mu) = \mathbf w + \mathbf A^\top \bm\lambda = \mathbf 0 \\
\nabla_bL(\mathbf{w,\bm \xi,}b;\bm \lambda,\bm \mu) = \mathbf y^\top \bm\lambda =0 \\
\nabla_{\bm\xi} L(\mathbf{w,\bm \xi,}b;\bm \lambda,\bm \mu) =2C\bm\xi -\bm \mu -\bm \lambda=\mathbf 0 \\
\bm\mu^\top \bm\xi= 0, \bm\lambda^\top \mathbf{(\bm\xi - Aw}+b\mathbf{y-e)}=0\\
\mathbf{\bm\xi \geq 0, \bm\xi-Aw+}b\mathbf{y-e\geq 0}\\
\mathbf{\bm\mu\geq 0, \bm\lambda \geq 0}
\end{array}\right.
\ee
The dual problem of (\ref{model:l2kernelv}) is
$$
\renewcommand\arraystretch{1.5}
\begin{array}{ll}
 \hbox{ }&\sup\limits_{\bm{\mathbf{\lambda\geq 0,\mu\geq 0}}}\inf\limits_{\mathbf w,\bm \xi, b}L(\mathbf w,\bm \xi,b;\bm \lambda,\bm \mu)\\
=&\sup\limits_{\mathbf{\bm\lambda\geq 0,\bm\mu\geq 0}}\left[\inf\limits_{\mathbf w}\left( \frac{1}{2}\|\mathbf w\|^2+\langle\mathbf A^\top \bm\lambda, \mathbf w\rangle \right)+  \inf\limits_{\bm\xi}\left(C\|\bm\xi\|^2- \langle\bm\xi, \bm\mu+\bm\lambda\rangle\right) +\inf\limits_b \left( (-\bm\lambda^\top \mathbf y)b +\langle\bm \lambda, \mathbf e\rangle\right)\right]\\
=&\sup\limits_{\mathbf{\bm\lambda\geq 0,\bm\mu\geq 0}} \left(\bm\lambda^\top \mathbf e -\frac{1}{2}\|\mathbf A^\top \bm\lambda\|^2 - \frac{1}{4C}\|\bm\mu+\bm\lambda\|^2\right)\\
=&\sup\limits_{\mathbf{\bm\lambda\geq 0}} \left(\bm\lambda^\top \mathbf e -\frac{1}{2}\|\mathbf A^\top \bm\lambda\|^2 - \frac{1}{4C}\|\bm\lambda\|^2\right),
\end{array}
$$
which is equivalent to

\be
\label{model:l2kerneldual}
\renewcommand\arraystretch{1.2}
\begin{array}{ll}
\min\limits_{\bm\lambda}
& \frac{1}{2}\|\mathbf A^\top  \bm\lambda\|^2+\frac{1}{4C}\|\bm\lambda\|^2 - \bm\lambda^\top \mathbf e\\
\hbox{s.t.} &\bm \lambda \geq \mathbf 0,\ \bm\lambda^\top \mathbf y = \mathbf 0.
\end{array}
\ee
Similarly we can derive the dual problem of (\ref{model:l2linear}) as
\be
\label{model:l2lineardual}
\renewcommand\arraystretch{1.2}
\begin{array}{ll}
\min\limits_{\bm\lambda}
& \frac{1}{2}\|\hat{\mathbf A}^\top  \bm\lambda\|^2+\frac{1}{4C}\|\bm\lambda\|^2 - \bm\lambda^\top \mathbf e\\
\hbox{s.t.} &\bm \lambda \geq \mathbf 0,\ \bm\lambda^\top \mathbf y = \mathbf 0,
\end{array}
\ee
where $\hat{\mathbf A}^\top := [-y_1\hat{\mathbf x}_1, \cdots, -y_n\hat{\mathbf x}_n]$.

By (\ref{eq:K}), we have
$$K_{ij}=\langle \bm \psi(\mathbf x_i), \bm \psi(\mathbf x_j)\rangle =\langle \hat{\mathbf x}_i, \hat{\mathbf x}_j \rangle, \ \forall i,j = 1,\dots, m+n,$$
which means
$$
\mathbf{A}\mathbf A^\top  =
\left[
\begin{array}{ccc}
y_1^2\bm\psi(\mathbf x_1)^\top \bm\psi(\mathbf x_1) &\cdots &y_1y_n\bm\psi(\mathbf x_1)^\top \bm\psi(\mathbf x_n)\\
\vdots &\ddots &\vdots\\
y_ny_1\bm\psi(\mathbf x_n)^\top \bm\psi(\mathbf x_1) &\cdots &y_ny_n\bm\psi(\mathbf x_n)^\top \bm\psi(\mathbf x_n)\\
\end{array}
\right]
=
\left[
\begin{array}{ccc}
y_1^2\bm\hat{\mathbf x}_1^\top \hat{\mathbf x}_1 &\cdots &y_1y_n\hat{\mathbf x}_1^\top \hat{\mathbf x}_n\\
\vdots &\ddots &\vdots\\
y_ny_1\bm\hat{\mathbf x}_n^\top \hat{\mathbf x}_1 &\cdots &y_ny_n\hat{\mathbf x}_n^\top \hat{\mathbf x}_n\\
\end{array}
\right]
=\hat{\mathbf A}\hat{\mathbf A}^\top
$$
and then
$$
\frac{1}{2}\|\mathbf A^\top  \bm\lambda\|^2 = \frac{1}{2}\|\hat{\mathbf A}^\top  \bm\lambda\|^2, \ \forall\ \bm \lambda.
$$
Therefore problem (\ref{model:l2kerneldual}) and problem (\ref{model:l2lineardual}) are equivalent.

Let $\bm\lambda^*$ be the optimal solution of both problem (\ref{model:l2kerneldual}) and problem (\ref{model:l2lineardual}). then by KKT condition we have
$$
\mathbf w^* = \mathbf A^\top  \bm \lambda^*, \ \hat{\mathbf w}^* = \hat{\mathbf A}^\top  \bm\lambda^*,
$$
and for any $\lambda_i^* \in (0, 2C)$,
$$
\renewcommand\arraystretch{1.5}
\begin{array}{ll}
b^* &=\frac{1}{y_i}\left( -\frac{1}{2C}\lambda_i^* -y_i\bm\psi(\mathbf x_i)^\top  \mathbf w^* +1\right)\\
& = -\frac{1}{2Cy_i}\lambda_i^* -\sum\limits_{j=1}^ny_j\bm\psi(\mathbf x_i)^\top \bm\psi(\mathbf x_j)  +\frac{1}{y_i}\\
& = -\frac{1}{2Cy_i}\lambda_i^* -\sum\limits_{j=1}^ny_j \hat{\mathbf x}_i^\top \hat{\mathbf x}_j  +\frac{1}{y_i}\\
& = \hat{b}^*.
\end{array}
$$
Hence for any $\mathbf x_i\in \mathbf X_e$ and corresponding $\hat{\mathbf x}_i\in \mathbf F_e$ we can write
$$
\renewcommand\arraystretch{1.5}
\begin{array}{ll}
sign((\mathbf w^*)^\top \bm\psi(\mathbf x_i)+b^*)
&=sign((\mathbf A^\top  \bm \lambda^*)^\top \bm\psi(\mathbf x_i)+ b^* ) \\
&=sign(-\sum\limits_{j=1}^n \lambda_j^* y_j \bm\psi(\mathbf x_j)^\top \bm\psi(\mathbf x_i)+ b^* ) \\
&=sign(-\sum\limits_{j=1}^n \lambda_j^* y_j \hat{\mathbf x}_j^\top \hat{\mathbf x}_i+ b^* ) \\
&=sign(-\sum\limits_{j=1}^n \lambda_j^* y_j \hat{\mathbf x}_j^\top \hat{\mathbf x}_i+ \hat{b}^* ) \\
&=sign((\hat{\mathbf A}^\top  \bm \lambda^*)^\top \bm\psi(\mathbf x_i)+ \hat{b}^* ) \\
&=sign((\hat{\mathbf w}^*)^\top \hat{\mathbf x}_i+\hat{b}^*),
\end{array}
$$
i.e. the predicting label in (\ref{eq:l2-svc-kernel}) and (\ref{model:l2linear}) are identical.

\end{proof}

\section{Proof of Theorem \ref{thm:bound_w}}\label{apd:proof_w}
Previous studies \citep{drineas2005approximating,cortes2010impact,zhang2012scaling} have given the bound of error of Nystr{\"o}m approximation, we present it in following lemma.
\begin{lemma}
\label{lem:ff}
Let $\mathbf{F}^{(k)}_{r}$ and $\mathbf{\widetilde F}_r$ be in (\ref{eq:fk}) and (\ref{eq:fkn}) respectively. We have
\be
\|\mathbf{F}^{(k)}_{r}- \mathbf{\widetilde F}_r\| \leq ke^{\frac{1}{4}}_f + \lambda_1tr(\mathbf A)+ke_2tr(\widetilde{\bm\Lambda}_{(k)}^{-1})\left(e_f^{\frac{1}{4}}+tr(\bm\Lambda_{(k)}^2)^{\frac{1}{4}}\right),
\ee
where $\mathbf A\in \mathbb R^{k\times k}$, $\bm\Lambda$ and $\widetilde{\bm\Lambda}$ are defined as in Theorem \ref{thm:bound_w}. 
\end{lemma}

We also need the following two lemmas, which are from \citep{cortes2010impact}. 
\begin{lemma} When training the kernel SVM  (\ref{eq:svc-l2-kernel}) on $\mathbf X_r$ with kernels $\psi$ and $\tilde\psi$ respectively, there is 
\label{lem:w_loss}
$$\|\mathbf w - \mathbf{\widetilde w}\|^2 \leq
\frac{C_0}{n}
\sum_{i=1}^n\left[
\left(L(y_i\mathbf{\widetilde w}^\top \bm\psi(\mathbf x_i))-L(y_i\mathbf{\widetilde w}^\top \bm{\widetilde \psi}(\mathbf x_i)) \right) +
\left(L(y_i\mathbf{w}^\top \bm{\widetilde \psi}(\mathbf x_i))-L(y_i\mathbf{w}^\top \bm\psi(\mathbf x_i)) \right)
\right],
$$
where $L(\cdot)$ is $a$ loss function.
\end{lemma}

\begin{lemma}
\label{lem:psi_f}
$$
\|\bm{\widetilde \psi}(\mathbf x)-\bm\psi(\mathbf x)\|
\leq \|\mathbf F_r^{(k)} -\mathbf{\widetilde F}_r\|,\ \forall\ \mathbf x\in\mathbf X_r.
$$
\end{lemma}

\begin{proposition}\label{prop:w_f}
%Let there exists $\rho \geq 0$,$G\geq 0$ such that $\max\{\kappa(\mathbf x,\mathbf x),\widetilde \kappa(\mathbf x,\mathbf x)\} \leq \rho$ and $\max\{|h(\mathbf x)|, |\widetilde{h}(\mathbf x)|\}\leq G$, for all $\mathbf x \in \mathbf X,h(\cdot)\in \mathbf{H}$ and $\widetilde h(\cdot)\in\mathbf{\widetilde H}$. 
Assume that Assumption \ref{assumption} holds. We have
$$
\|\mathbf w - \mathbf{\widetilde w}\|^2 \leq 4C_0^2G(G+1)\rho^{\frac{1}{2}}\|\mathbf F_r^{(k)} -\mathbf{\widetilde F}_r\|.
$$
\end{proposition}

\begin{proof}
Recall that $C=C_0/n$. Firstly we use the mean-value theorem on the function $g(z) = z^2$ between $z_1 := [y_i\mathbf{w'}^\top \bm\psi(\mathbf x_i)]_+$ and $z_2:=[y_i\mathbf{\widetilde w}^\top \bm{\widetilde \psi}(\mathbf x_i)]_+$, we can get
$$
\renewcommand\arraystretch{1.8}
\begin{array}{l}
\left|[y_i\mathbf{\widetilde w}^\top \bm\psi(\mathbf x_i)]_+^2-[y_i\mathbf{\widetilde w}^\top \bm{\widetilde \psi}(\mathbf x_i)]_+^2 \right| \\
=2\theta \left| [y_i\mathbf{\widetilde w}^\top \bm\psi(\mathbf x_i)]_+ - [y_i\mathbf{\widetilde w}^\top \bm{\widetilde \psi}(\mathbf x_i)]_+\right|\\
\leq 2\theta \left| y_i\mathbf{\widetilde w}^\top \bm\psi(\mathbf x_i) - y_i\mathbf{\widetilde w}^\top \bm{\widetilde \psi}(\mathbf x_i)\right|\\
\leq 2\max\{|\mathbf{\widetilde w}^\top \bm\psi(\mathbf x_i)|,|\mathbf{\widetilde w}^\top \bm{\widetilde \psi}(\mathbf x_i)|\}\left| y_i\mathbf{\widetilde w}^\top \bm\psi (\mathbf x_i) - y_i\mathbf{\widetilde w}^\top \bm{\widetilde \psi}(\mathbf x_i)\right|,
\end{array}
$$
where $\theta$ is between $[y_i\mathbf{\widetilde w}^\top \bm\psi(\mathbf x_i)]_+$ and $[y_i\mathbf{\widetilde w}^\top \bm{\widetilde \psi}(\mathbf x_i)]_+$.

With (\ref{ass-2}) in Assumption \ref{assumption}, we have
$$
2\max\{[\mathbf{\widetilde w}^\top \bm\psi(\mathbf x_i)]_+,\ [\mathbf{\widetilde w}^\top \bm{\widetilde \psi}(\mathbf x_i)]_+\} \leq 2G.
$$
Therefore
$$
|[y_i\mathbf{\widetilde w}^\top \bm\psi(\mathbf x_i)]_+^2-[y_i\mathbf{\widetilde w}^\top \bm{\widetilde \psi}(\mathbf x_i)]_+^2 | \leq 2G| y_i\mathbf{\widetilde w}^\top \bm\psi(\mathbf x_i) - y_i\mathbf{\widetilde w}^\top \bm{\widetilde \psi}(\mathbf x_i)|, \ \forall\ i, \forall\ \mathbf{\widetilde w}\in \mathbf W.
$$
Similarly we have
$$
|[y_i\mathbf{w}^\top \bm{\widetilde \psi}(\mathbf x_i)]_+^2-[y_i\mathbf{w}^\top \bm\psi(\mathbf x_i)]_+^2| \leq 2G| y_i\mathbf{w}^\top \bm{\widetilde \psi}(\mathbf x_i) - y_i\mathbf{w}^\top \bm\psi(\mathbf x_i)|, \ \forall\ i, \forall\ \mathbf{w}\in \mathbf W.
$$
Then by Lemma \ref{lem:w_loss} we can write
$$
\renewcommand\arraystretch{1.5}
\begin{array}{ll}
\|\mathbf w - \mathbf{\widetilde w}\|^2
&\leq
\frac{C_0}{n}
\sum_{i=1}^n\left[
([y_i\mathbf{\widetilde w}^\top \bm\psi(\mathbf x_i)]_+^2-[y_i\mathbf{\widetilde w}^\top \bm{\widetilde \psi}(\mathbf x_i)]_+^2) +
([y_i\mathbf{w}^\top \bm{\widetilde \psi}(\mathbf x_i)]_+^2-[y_i\mathbf{w}^\top \bm\psi(\mathbf x_i)]_+^2 )
\right]\\
&\leq
\frac{2C_0G}{n}
\sum_{i=1}^n\left(
\left|\mathbf{\widetilde w}^\top \bm\psi(\mathbf x_i)-\mathbf{\widetilde w}^\top \bm{\widetilde \psi}(\mathbf x_i)\right|+
\left|\mathbf{w}^\top \bm{\widetilde \psi}(\mathbf x_i)-\mathbf{w}^\top \bm\psi(\mathbf x_i) \right|
\right)\\
&\leq
\frac{2C_0G}{n}
\sum_{i=1}^n\left(
\|\mathbf{\widetilde w}\|\|\bm\psi(\mathbf x_i)-\bm{\widetilde \psi}(\mathbf x_i)\|+
\|\mathbf w\|\|\bm{\widetilde \psi}(\mathbf x_i)-\bm\psi(\mathbf x_i) \|
\right)\\
&=
\frac{2C_0G}{n}
\sum_{i=1}^n (\|\mathbf{\widetilde w}\|+\|\mathbf w\|)\left(
\|\bm\psi(\mathbf x_i)-\bm{\widetilde \psi}(\mathbf x_i)\|+
\|\bm{\widetilde \psi}(\mathbf x_i)-\bm\psi(\mathbf x_i) \|
\right).\\
\end{array}
$$
Weight vector $\mathbf w$ can be written in terms of dual variables: $\mathbf w = \sum_{i=1}^n y_i\bm\psi(\mathbf x_i)\lambda_i$ where $\lambda_i$ is the dual variable. By the KKT condition, which is similar to (\ref{KKT}), we have $2C\xi_i = \lambda_i$ and $\lambda_i(\xi_i +y_i\mathbf w^\top\bm\psi(\mathbf x_i)-1)=0, \ \forall \ i$. Then we get $\lambda_i$ is either $0$ or $2\frac{C_0}{n}(1-y_i\mathbf w^\top  \bm\psi(\mathbf x_i))$. Hence
$$
\renewcommand\arraystretch{1.5}
\begin{array}{rl}
\|\mathbf w\| &= \|\sum_{i=1}^n y_i\bm\psi(\mathbf x_i)\lambda_i\| \\
& \leq\sum_{i=1}^n \|y_i\|\|\bm\psi(\mathbf x_i)\|\|\lambda_i\| \\
& = \sum_{i=1}^n \|\bm\psi(\mathbf x_i)\|\|\lambda_i\| \\
&\leq 2\frac{C_0}{n}\sum_{i=1}^n \|\bm\psi(\mathbf x_i)\|(\|\mathbf w^\top  \bm\psi(\mathbf x_i)\| +1) \\
&\leq 2\frac{C_0}{n}(G+1)\sum_{i=1}^n \|\bm\psi(\mathbf x_i)\|
\end{array}
$$

Due to  (\ref{ass-2}) in Assumption \ref{assumption}, there is \[\max\{\|\bm\psi(\mathbf x)\|, \|\bm{\widetilde \psi}(\mathbf x)\|\} = \max\{\kappa(\mathbf x,\cdot),\ \widetilde \kappa(\mathbf x,\cdot)\}(or \sqrt{\max\{\kappa(\mathbf x,x),\widetilde \kappa(\mathbf x,x)\}} )\leq \rho^{\frac{1}{2}}, \ \forall\ \mathbf x\in\mathbf X.
\] Then
\begin{equation*}
\|\mathbf w\| \leq 2C_0(G+1)\frac{1}{n}\sum_{i=1}^n \|\bm\psi(\mathbf x_i)\|\leq 2C_0(G+1)\rho^{\frac{1}{2}},
\end{equation*}
similarly we have
\begin{equation*}
\|\mathbf{\widetilde w}\| \leq 2C_0(G+1)\frac{1}{n}\sum_{i=1}^n \|\bm\psi(\mathbf x_i)\|\leq 2C_0(G+1)\rho^{\frac{1}{2}}.
\end{equation*}
Additionally, with Lemma \ref{lem:psi_f}, we can get
\begin{equation*}
\begin{array}{rl}
\|\mathbf w - \mathbf {\widetilde w}\|^2
&\leq
2C_0G(\|\mathbf{\widetilde w}\|+\|\mathbf w\|)
\frac{1}{n}\sum_{i=1}^n \left(
\|\bm\psi(\mathbf x_i)-\bm{\widetilde \psi}(\mathbf x_i)\|+
\|\bm{\widetilde \psi}(\mathbf x_i)-\bm\psi(\mathbf x_i) \|
\right)\\
&\leq
4C_0^2G(G+1)\rho^{\frac{1}{2}}\frac{1}{n}\sum_{i=1}^n \left(
\|\bm\psi(\mathbf x_i)-\bm{\widetilde \psi}(\mathbf x_i)\|+
\|\bm{\widetilde \psi}(\mathbf x_i)-\bm\psi(\mathbf x_i) \|
\right)\\
&\leq  4C_0^2G(G+1)\rho^{\frac{1}{2}}\|\mathbf F_r^{(k)} -\mathbf{\widetilde F}_r\|.
\end{array}
\end{equation*}
which is desired result.
\end{proof}

\paragraph{Proof of Theorem \ref{thm:bound_w}} Combining Lemma \ref{lem:ff} and Proposition \ref{prop:w_f},  we have
$$
\|\mathbf w - \mathbf{\widetilde w} \|^2 \leq 4C_0^2G(G+1)\rho^{\frac{1}{2}}\left[ke^{\frac{1}{4}}_f + \lambda_1tr(\mathbf A)+ke_2tr(\widetilde{\bm\Lambda}_{(k)}^{-1})\left(e_f^{\frac{1}{4}}+tr(\bm\Lambda_{(k)}^2)^{\frac{1}{4}}\right)\right],
$$
which gives the bound of $\|\mathbf w - \mathbf{\widetilde w}\|$. The proof is finished. \hfill$\Box$

%\begin{acknowledgements}
%\end{acknowledgements}


%\bibliographystyle{plainnat}
\bibliographystyle{chicago}
\bibliography{LASNref}   % name your BibTeX data base

%% Non-BibTeX users please use
%\begin{thebibliography}{}
%%
%% and use \bibitem to create references. Consult the Instructions
%% for authors for reference list style.
%\end{thebibliography}

\end{document}